\documentclass{article}

\usepackage{graphicx} 
\usepackage{subfigure}

\usepackage{natbib}

\usepackage{algorithm}
\usepackage{algorithmic}

\usepackage{hyperref}

\usepackage[accepted]{icml2011}

\usepackage{multirow}

\usepackage{amsmath,amssymb,amsthm,bm,multicol,longtable,float,dsfont}

\newtheorem{lemma}[subsection]{Lemma}
\newtheorem{theorem}[subsection]{Theorem}

\newtheorem{definition}[subsection]{Definition}

\icmltitlerunning{Data-Distributed Weighted Majority and Online Mirror Descent}

\begin{document}

\twocolumn[
\icmltitle{Data-Distributed Weighted Majority and Online Mirror Descent}

\icmlauthor{}{} 
\icmladdress{{\bf Keywords}: online learning, stochastic approximation,
             mirror descent, parallel computing, large scale} 

\icmlauthor{Hua Ouyang}{houyang@cc.gatech.edu} \icmladdress{College of Computing, Georgia Institute of Technology}
\icmlauthor{Alexander Gray}{agray@cc.gatech.edu} \icmladdress{College of Computing, Georgia Institute of Technology}

\vskip 0.3in
]

\begin{abstract}
In this paper, we focus on the question of the extent to which online learning can benefit from distributed computing. We focus on the setting in which $N$ agents online-learn cooperatively, where each agent only has access to its own data. We propose a generic data-distributed online learning meta-algorithm. We then introduce the Distributed Weighted Majority and Distributed Online Mirror Descent algorithms, as special cases. We show, using both theoretical analysis and experiments, that compared to a single agent: given the same computation time, these distributed algorithms achieve smaller generalization errors; and given the same generalization errors, they can be $N$ times faster.
\end{abstract}

\section{Introduction}
\label{sec:intro}
The real world can be viewed as a gigantic distributed system that evolves over time. An intelligent agent in this system can learn from two sources: examples from the environment, as well as information from other agents.  One way to state the question addressed by the {\em Data-Distributed Online Learning} (DDOL) schemes we introduce can be informally described as follows:  within an interconnected network of learning agents, although an agent only receives $m$ samples of input data, can it be made to perform as if it has received $M > m$ samples?  Here the performance is measured by generalization abilities (prediction error or regret for the online setting). In other words, to what extent can an agent make fewer generalization errors by utilizing information from other online-learning agents?

This question can also be phrased another way. In recent years, the increasing ubiquity of massive datasets as well as the opportunities for distributed computing (cloud computing, multi-core, etc.), have conspired to spark much interest in developing distributed algorithms for machine learning (ML) methods. While it is easy to see how parallelism can be obtained for most of the computational problems in ML, the question arises whether online learning, which appears at first glance to be inherently serial, can be fruitfully parallelized to any significant degree. While several recent papers have proposed distributed schemes, the question of whether significant speedups over the default serial scheme can be achieved has remained fairly open. Theory establishing or disallowing such a possibility is particularly to be desired. To the best of our knowledge, this paper is the first work that answers these questions for the general online learning setting.

In this paper we show both theoretically and experimentally that significant speedups are possible in online learning by utilizing parallelism.  We introduce a general framework for data-distributed online learning which encapsulates schemes such as weighted majority, online subgradient descent, and online exponentiated gradient descent.

%
\subsection{Related Work}
In an empirical study \cite{delalleau07}, the authors proposed to make a trade-off between batch and stochastic gradient descent by using averaged mini-batches of size $10\sim 100$. A parameter averaging scheme was proposed in \cite{mann09} to solve a batch regularized conditional max entropy optimization problem, where the distributed parameters from each agent is averaged in the final stage. A distributed subgradient descent method was proposed in \cite{nedic09dsmmao} and an incremental subgradient method using a Markov chain was proposed in \cite{johnsson09rismdons}. In \cite{duchi10}, a distributed dual averaging algorithm was proposed for minimizing convex empirical risks via decentralized networks. Convergence rates were reported for various network topologies. The same idea of averaged subgradients was extended to centralized online learning settings in \cite{dekel10}. The problem of multi-agent learning has been an active research topic in reinforcement learning. In this paper, we will focus on the online supervised learning setting.
%

\section{Setup and DDOL Meta-Algorithm}
In this paper, we assume that each agent only has access to a portion of the data locally and communications with other agents. Suppose we have $N$ learning agents. At round $t$, the $i^{\text{th}}$ learning agent $\mathcal{A}_i : i=1,\ldots, N$ receives an example $\mathbf{x}_i^t\in\mathbb{R}^D$ from the environment and makes a prediction $y_i^t$. The environment then reveals the correct answer $l_i^t$ corresponding to $\mathbf{x}_i^t$ and the agent suffer some loss $L(\mathbf{x}_i^t, y_i^t, l_i^t)$. The parameter set of an agent $\mathcal{A}_i$ at time $t$ is $\mathbf{w}_i^t \in \mathcal{W}$. Each agent is a vertex of a connected graph $G=(\mathcal{A},\mathcal{E})$. There will be a bidirectional communication between $\mathcal{A}_i$ and $\mathcal{A}_j$ if they are neighbors (connected by an edge $e_{ij}$). $\mathcal{A}_i$ has $N_i - 1$ neighbors.

The generic meta-algorithm for data-distributed online learning (\textsf{DDOL}) is very simple: each agent $\mathcal{A}_i$ works according to the following procedure:
\begin{algorithm}
\caption{\textsf{Distributed Online Learning (DDOL)}}
\label{alg:DOL}
\begin{algorithmic}
\FOR{$t=1,2,\ldots$}
  \STATE $\mathcal{A}_i$ makes local prediction(s) on example(s) $\mathbf{x}_i^t$;
  \STATE $\mathcal{A}_i$ \textsf{Update} $\mathbf{w}_i^t$ using local correct answer(s) $l_i^t$;
  \STATE $\mathcal{A}_i$ \textsf{Communicate} $\mathbf{w}_i^t$ with its neighbors and do \textsf{Weighted Average} over $\mathbf{w}_j^t,\ j=1,\ldots,N_i$;
\ENDFOR
\end{algorithmic}
\end{algorithm}

To derive a distributed online learning algorithm, one need to specify the \textsf{Update}, \textsf{Communicate} and \textsf{Weighted Average} schemes. In the following sections, we will discuss how to use two classic online learning methods as the basic \textsf{Update} scheme, and how the combination with different \textsf{Communicate/Weighted Average} schemes leads to different performance guarantees.

\section{Distributed Weighted Majority}\label{sec:dwm}
We firstly propose two expert-advise-based online classification algorithms which can be regarded as distributed versions of the classic \emph{Weighted Majority algorithm} (WMA) \cite{wma89littlestone}. For simplicity, we assume that in both algorithms, all the experts are shared by all the agents, and each agent is adjacent to \emph{all} the other agents ($G$ is a complete graph).

Alg. \ref{alg:dist_wm_imit} is named \emph{Distributed Weighted Majority by Imitation} (\textsf{DWM-I}). In the communication step, each $\mathcal{A}_i$ \textbf{mimics} other agent's operations by penalizing each expert $p$ in the same way as any other agents do, and then makes a geometric averaging.
\begin{algorithm}
\caption{\textsf{DWM-I}: agent $\mathcal{A}_i$	} \label{alg:dist_wm_imit}
\begin{algorithmic}[1]
\STATE Initialize all weights $w_{i_1}^0,\ldots,w_{i_P}^0$ of $P$ shared experts for agent $\mathcal{A}_i$ to $1$.
\FOR{$t=1,2,\ldots$}
    \STATE Given experts' predictions $y_{i_1}^t,\ldots,y_{i_P}^t$ over $\mathbf{x}_i^t$, $\mathcal{A}_i$ predicts
    \STATE $\begin{cases}
            1, &  \text{if }\sum_{p:y_{i_p}=1}w_{i_p}^{t-1} \geq \sum_{p:y_{i_p}=0}w_{i_p}^{t-1}\\
            -1, & \text{otherwise.}
            \end{cases}$
    \STATE Environment reveals $l_{i}^t$ for $\mathcal{A}_i$.
    \STATE $\forall p:\widetilde{w}_{i_p}^t \leftarrow \begin{cases}
            \alpha w_{i_p}^{t-1}, &  \text{if }y_{i_p}^t \neq l_{i}^t \text{ ($0 < \alpha < 1$)}\\
            w_{i_p}^{t-1}, & \text{otherwise.}
            \end{cases}$
    \STATE $w_{i_p}^t \leftarrow \left( \prod_{j=1}^N \widetilde{w}_{j_p}^t \right )^{1/N}$.
\ENDFOR
\end{algorithmic}
\end{algorithm}
The following result gives the upper bound of the average number of mistakes by each agent, assuming that each agent is receiving information from all the other agents.
\begin{theorem}\label{thm:dist_wm_imit}
For Algorithm \ref{alg:dist_wm_imit} with $N$ agents and $P$ shared experts, $\max_i M_i \leq \frac{1}{\log\frac{2}{1+\alpha}}\big( \frac{m_{*}}{N}\log\frac{1}{\alpha}+\log P \big)$, where $M_i$ is the number of mistakes by agent $\mathcal{A}_i$ and $m_{*}$ is the minimum number of mistakes by the best expert over all agents so far.
\end{theorem}
\begin{proof}
The proof essentially follows that of WMA. The best expert $E_*$ makes $m_*$ mistakes over all agents so far. So for any $\mathcal{A}_i$, its weight of $E_*$ is $\alpha^{\frac{m_*}{N}}$. Upon each mistake made by any agent, the total weights $\sum_p w_{i_p}^t$ of $\mathcal{A}_i$ decreases by a factor of at least $\frac{1}{2}(1-\alpha)$. So the total weights for $\mathcal{A}_i$ is at most $P\left[ 1-\frac{1}{2}(1-\alpha)  \right]^{M_i}$. Therefore for any $i$, $\alpha^{\frac{m_*}{N}} \leq P\left( \frac{1+\alpha}{2} \right)^{M_i}$.
It follows that
\begin{equation}\label{eq:dist_wm_imit_1}
M_i \leq \frac{1}{\log\frac{2}{1+\alpha}}\left( \frac{m_{*}}{N}\log\frac{1}{\alpha}+\log P \right).
\end{equation}
Taking $\alpha = 1/2$,
\begin{equation}\label{eq:dist_wm_imit_2}
M_i < 2.41\left( \frac{m_{*}}{N}+\log P \right).
\end{equation}
\end{proof}
Comparing (\ref{eq:dist_wm_imit_2}) with the result of WMA: $M<2.41(m_*+\log P)$, in the most optimistic case, if $m_*$ is of the same order as the number of mistakes made by the best expert $E_*$ in the \emph{single} agent scheme, in other words, if $E_*$ makes $0$ error over all agents other than $\mathcal{A}_i$, then the upper bound is decreased by a factor of $1/N$. In the most pessimistic case, if $E_*$ makes exactly the same number of mistakes over \emph{all} $N$ agents, then the upper bound is the same as with a single agent. This happens when all agents are receiving the same inputs and answers from the environment, hence there is no new information being communicated among the network and no communications are needed. In reality, $m_*$ falls between these two extremes.
%

Theorem \ref{thm:dist_wm_imit} is stated from an individual agent point of view. From the social point of view, the total number of mistakes made by all agents $\sum_{i=1}^N M_i$ is upper bounded by 
\begin{equation}\label{eq:dwm_total_bound}
\frac{1}{\log\frac{2}{1+\alpha}}\bigg( m_{*}\log\frac{1}{\alpha}+ N\log P\bigg ),
\end{equation}
which is not larger than that in a single agent scheme ($N\log P$ can be ignored in comparing with the first term $m_*$ which could be very large in practice). Imagine that $NT$ samples are processed in the single agent scheme, while in the $N$ agents scheme, each $\mathcal{A}_i$ process $T$ samples. In the most pessimistic case, upper bound (\ref{eq:dwm_total_bound}) is the same for these two schemes. This is a very good property for parallel computing, since the proposed online DWM can achieve the same generalization capacity, while being $N$ times faster than a serial algorithm. This property is verified by the experiments in Section \ref{sec:exp}.

As in the Randomized Weighted Majority algorithm (RWM) \cite{wma89littlestone}, we can introduce some randomness to our choices of experts by giving each expert a probability of being chosen depending on the performance of this expert in the past. Specifically, in each round we choose an expert with probability $p_i = w_i/\sum_i w_i$. We can have a Distributed Randomized Weighted Majority and obtain a similar upper bound as that of RWM with a constant of $1/N$ as in Theorem \ref{thm:dist_wm_imit}.

The upper bound (\ref{eq:dist_wm_imit_1}) can be further improved by an alternative algorithm (Alg. \ref{alg:dist_wm_avg}), which differs from Alg. \ref{alg:dist_wm_imit} only in the way that each agent utilizes information received from others. Instead of mimicking other agents' operations, an agent now updates its weights by arithmetically \textbf{averaging} together with all the weights it received from its neighbors.
\begin{algorithm}
\caption{\textsf{DWM-A}: agent $\mathcal{A}_i$} \label{alg:dist_wm_avg}
\begin{algorithmic}[1]
\STATE $\cdots$
\STATE $w_{i_p}^t \leftarrow \frac{1}{N}\sum_{j=1}^N \widetilde{w}_{j_p}^t$.
\STATE $\cdots$
\end{algorithmic}
\end{algorithm}

\begin{theorem}\label{thm:dist_wm_avg}
For Algorithm \ref{alg:dist_wm_avg} with $N$ agents and $P$ shared experts, $\max_i{M_i} \leq \frac{1}{\log\frac{2}{1+\alpha}}\left( \sum_t \frac{(1-\alpha)m_*^t}{N-(1-\alpha)m_*^t} +\log P \right)$, where $M_i$ is the number of mistakes by agent $\mathcal{A}_i$ so far and $m_*^t$ is the minimum number of mistakes by the best expert at round $t$ over all agents.
\end{theorem}
\begin{proof}
Denote the weight of expert $p$ for agent $\mathcal{A}_i$ at round $t$ as $w_{i_p}^t$. Indeed,
\begin{equation*}
\begin{split}
&\sum_{i=1}^N w_{i_p}^t = \bigg(\sum_{i=1}^N w_{i_p}^{t-1} \bigg) \bigg( \frac{N-m_p^t}{N}+\alpha\frac{m_p^t}{N} \bigg ) \\
&= \bigg(\sum_{i=1}^N w_{i_p}^{t-2} \bigg) \bigg[ 1-\frac{m_p^t(1-\alpha)}{N}\bigg] \bigg[ 1-\frac{m_p^{t-1}(1-\alpha)}{N}\bigg]\\
&= \cdots = N \bigg[ 1-\frac{m_p^t(1-\alpha)}{N}\bigg] \cdots \bigg[ 1-\frac{m_p^{1}(1-\alpha)}{N}\bigg].
\end{split}
\end{equation*}
Using $1-x \geq \exp(-x/(1-x)),\ \forall x\in(0,1)$ and the fact that $0\leq m_p^t \leq N$, we have for any agent $\mathcal{A}_i$,
\begin{equation}\label{eq:dist_wm_avg_1}
\begin{split}
&w_{i_p}^t = \prod_t \bigg[1-\frac{m_p^t(1-\alpha)}{N} \bigg]
\geq \exp\bigg(\sum_{t} \frac{-m_p^t(1-\alpha)}{N-m_p^t(1-\alpha)} \bigg).
\end{split}
\end{equation}
On the other hand, for any $\mathcal{A}_i$,
\begin{equation}\label{eq:dist_wm_avg_2}
\sum_{p=1}^P w_{i_p}^t \leq P\left[ 1- \frac{1}{2}(1-\alpha)\right]^{M_i}.
\end{equation}
Since $w_{i_p}^t \leq \sum_{p=1}^P w_{i_p}^t$, combining (\ref{eq:dist_wm_avg_1}) and (\ref{eq:dist_wm_avg_2}),
\begin{equation*}
P\left[ 1- \frac{1}{2}(1-\alpha)\right]^{M_i} \geq \exp\bigg( -\sum_t \frac{m_p^t(1-\alpha)}{N-m_p^t(1-\alpha)} \bigg).
\end{equation*}
It follows that $\forall i=1\ldots N,\ p=1\ldots P$
\begin{equation}\label{eq:dist_wm_avg_3}
M_i \leq \frac{1}{\log\frac{2}{1+\alpha}}\bigg( \sum_t \frac{(1-\alpha)m_p^t}{N-(1-\alpha)m_p^t} +\log P \bigg).
\end{equation}
\end{proof}

Now we are ready to compare the refined bound (\ref{eq:dist_wm_avg_3}) with (\ref{eq:dist_wm_imit_1}) using $m_* \geq \sum_t m_*^t$. Without considering the $\log P$ part of the bounds which is much smaller than the $m_*$ part, it is easy to verify that if $1/2\leq \alpha < 1$, then
\begin{equation}\label{eq:dist_wm_avg_4}
\frac{m_* \log\frac{1}{\alpha}}{N} > \sum_t \frac{(1-\alpha)m_*^t}{N-(1-\alpha)m_*^t}
\end{equation}
without any assumption on $m_*^t$; If $0< \alpha < 1/2$, then the above inequality holds when
\begin{equation}\label{eq:dist_wm_avg_5}
m_*^t \leq N\left(\frac{1}{1-\alpha} - \frac{1}{\log(1/\alpha)} \right).
\end{equation}
The RHS of (\ref{eq:dist_wm_avg_5}) is lower bounded by $0.81N$.
Specifically, when $m_*^t = O(N/2)$ and by taking $\alpha = 1/2$, the difference in (\ref{eq:dist_wm_avg_4}) is
\begin{equation*}
\frac{m_*}{N} - \sum_t \frac{m_*^t}{2N-m_*^t} = O\left(\frac{m_*}{N}\right).
\end{equation*}
Hence the error bound in Theorem $\ref{thm:dist_wm_avg}$ is much lower. Experimental evidence will be provided in Section \ref{sec:exp}.


\section{Distributed Online Mirror Descent}\label{sec:DOMD}
In this section we extend the idea of distributed online learning to Online Convex Optimization (OCO) problems. OCO is an online variant of the convex optimization, which is ubiquitous in many machine learning problems such as support vector machines, logistic regression and sparse signal reconstruction tasks. Each of these learning algorithms has a convex loss function to be minimized.

One can consider OCO as a repeated game between an algorithm $\mathcal{A}$ and the environment. At each round $t$, $\mathcal{A}$ chooses a strategy $\mathbf{w}_t\in \mathcal{W}$ and the environment reveals a convex function $f_t$. Here we assume that all convex functions share the same feasible set $\mathcal{W}$. The goal of $\mathcal{A}$ is to minimize the difference between the cumulation $\sum_t f_t(\mathbf{w}_t)$ and that of the best strategy $\mathbf{w}^*$ it can play in hindsight. This difference is commonly known as \emph{external regret}, defined as below.
\begin{definition}
The regret of convex functions $\mathbf{f} = \{f_t\}$ for $t=1,2,\ldots,T$ is defined as $R(T) = \sum_{t=1}^T f_t(\mathbf{w}_t) - \inf_{\mathbf{w}\in\mathcal{W}}\sum_{t=1}^T f_t(\mathbf{w})$.
\end{definition}
In the distributed setting, this game is played by every agent $\mathcal{A}_i,\ i=1,\ldots, N$. The goal of $\mathcal{A}_i$ is to minimize its own regret $R_i(T)$, named the \emph{individual regret}. We call the sum of individual regrets $R(T) = \sum_{i=1}^N R_i(T)$ \emph{social regret}.

We will present the online mirror descent (OMD) framework which generalize many OCO algorithms such as online subgradient descent \cite{zinkevich03ocp}, Winnow \cite{littlestone88winnow}, online exponentiated gradient \cite{kivinen97eg}, online Newton's method \cite{hazan06lraoco}.

We firstly introduce some notations used in this section. A distance generating function $\omega(\mathbf{u})$ is a continuously differentiable function that is $a$-strongly convex w.r.t. some norm $\|\cdot\|$ associated with an inner product $\left\langle \cdot, \cdot \right\rangle$. Using Bregman divergence $\psi(\mathbf{u},\mathbf{v})= \omega(\mathbf{u})-\omega(\mathbf{v}) - \left\langle \nabla\omega(\mathbf{v}), \mathbf{u}-\mathbf{v} \right\rangle$ as a proximity function, the update rule of OMD can be expressed as
$\mathbf{w}_{t+1} \leftarrow \arg\min_{\mathbf{z}\in \mathcal{W}} \eta_t \left\langle \mathbf{g}_t, \mathbf{z}-\mathbf{w}_t \right\rangle + \psi(\mathbf{z},\mathbf{w}_t)
$, where $\mathbf{g}_t$ is a subgradient of $f_t$ at $\mathbf{w}_t$ and $\eta_t$ is a learning rate which plays an important role in the regret bound. Denote the dual norm of $\|\cdot\|$ as $\|\cdot\|_*$.

Suppose agent $\mathcal{A}_i$ has $N_i - 1$ neighbors. We propose Distributed Online Mirror Descent algorithm in Alg. \ref{alg:dist_domd}. In this algorithm, the update rule has explicit expressions for some special proximity functions $\psi(\cdot,\cdot)$. Next we derive distributed update rules for two well-known OMD examples: Online Gradient Descent (OGD) and Online Exponentiated Gradient (OEG).
\begin{algorithm}[h!]
\caption{\textsf{DOMD}: agent $\mathcal{A}_i$} \label{alg:dist_domd}
\begin{algorithmic}[0]
\STATE Initialize $w_{i}^1 \in \mathcal{W}$
\FOR{$t=1,2,\ldots$}
	\STATE Local prediction using $\mathbf{w}_i^t$.
    \STATE $\mathbf{w}_i^{t+1} \leftarrow \arg\displaystyle\min_{\mathbf{z}\in \mathcal{W}} \sum_{j=1}^{N_i} \bigg[ \eta_t \left\langle \mathbf{g}_j^t, \mathbf{z}-\mathbf{w}_j^t \right\rangle + \psi(\mathbf{z},\mathbf{w}_j^t) \bigg]$.
\ENDFOR
\end{algorithmic}
\end{algorithm}

\subsection*{Distributed OGD}
Taking $\psi(\mathbf{u},\mathbf{v}) = \frac{1}{2}\|\mathbf{u}-\mathbf{v}\|_2^2$ (i.e. the proximity is measured by squared Euclidean distance), an agent $\mathcal{A}_i$ needs to solve the minimization
$\min_{\mathbf{z}} \sum_{i=1}^{N_i}\left[ \eta_t \left\langle \mathbf{g}_i^t, \mathbf{z}-\mathbf{w}_i^t \right\rangle + \frac{1}{2}\|\mathbf{z}-\mathbf{w}_i^t\|_2^2 \right]
$,
which leads to a simple \textsf{DOGD} updating rule
\begin{equation}\label{eq:DOGD}
\mathbf{w}_i^{t+1} \leftarrow \frac{1}{N_i} \sum_{j=1}^{N_i} \bigg( \mathbf{w}_j^t - \eta_t \mathbf{g}_j^t \bigg).
\end{equation}

\subsection*{Distributed OEG}
Taking the unnormalized relative entropy as the proximity function $\psi(\mathbf{u},\mathbf{v}) = \sum_{d=1}^D u_d\ln u_d - \sum_{d=1}^D v_d\ln v_d - (\ln \mathbf{v} + I)^T(\mathbf{u}-\mathbf{v})$, we can solve the minimization $\min_{\mathbf{z}} \sum_{j=1}^{N_i} \big[ \eta_t \left\langle \mathbf{g}_j^t, \mathbf{z}-\mathbf{w}_j^t \right\rangle + \sum_{d=1}^D z_d\ln z_d - \sum_{d=1}^D (\mathbf{w}_j^t)_d \ln (\mathbf{w}_j^t)_d - (\ln \mathbf{w}_j^t + I)^T(\mathbf{z}-\mathbf{w}_j^t) \big]$, and obtain the update rule for \textsf{DOEG}:
\begin{equation}\label{eq:DOEG}
\mathbf{w}_i^{t+1} \leftarrow \bigg ( \prod_{j=1}^{N_i} \mathbf{w}_j^t e^{-\eta_t \mathbf{g}_j^t} \bigg )^{1/N_i}.
\end{equation}
If the feasible set $\mathcal{W}$ is a simplex ball $\|\mathbf{w}\|_1 \leq S$ instead of $\mathbb{R}^D$, one only needs to do an extra normalization: $\forall d=1,\ldots, D, (\mathbf{w}_i^{t+1})_d \leftarrow \frac{S (\mathbf{w}_i^{t+1})_d}{\sum_d (\mathbf{w}_i^{t+1})_d}$ if $\|\mathbf{w}_i^{t+1}\|_1>S$.

Updating rules (\ref{eq:DOGD}) and (\ref{eq:DOEG}) share the same spirit as stated in the meta-algorithm \ref{alg:DOL}: each agent updates its parameters $\mathbf{w}_i$ individually, then it averages with its neighbors' parameters, either arithmetically, or geometrically. The following results shows how this simple averaging scheme works, in terms of average \emph{individual regrets} $\frac{1}{N}\sum_i R_i(T)$. As in theorem \ref{thm:dist_wm_imit} and theorem \ref{thm:dist_wm_avg}, for simplicity, we assume that the graph $G$ is complete, i.e. each agent has $N-1$ neighbors.
\begin{lemma}\label{lem:dol}\cite{nemirovski83pcmeo}
Let $P_{\mathbf{w}}(\mathbf{u}) = \arg\min_{\mathbf{z}\in\mathcal{W}} \left\langle \mathbf{u}, \mathbf{z}- \mathbf{w} \right\rangle + \psi(\mathbf{z}, \mathbf{w}) $, for any $\mathbf{v},\mathbf{w}\in\mathcal{W}$ and $\mathbf{u} \in \mathbb{R}^D$ one has
\begin{equation}\label{eq:lem_dol}
\left\langle \mathbf{u}, \mathbf{w}-\mathbf{v} \right\rangle \leq \psi(\mathbf{w},\mathbf{v}) - \psi(P_{\mathbf{w}}(\mathbf{u}), \mathbf{v}) + \frac{\|\mathbf{u}\|_*^2}{2a}.
\end{equation}
\end{lemma}
\begin{theorem}\label{thm:dol}
If $N$ agents in Algorithm \ref{alg:dist_domd} are connected via a complete graph, $f_t$ are convex, distances between two parameter vectors are upper bounded $\sup_{i,j,t}\psi(\mathbf{w}_i^t,\mathbf{w}_j^t) = D$, let $\eta_t = \frac{1}{\sqrt{t}}$, then the average individual regret
\begin{equation}\label{eq:domd_bound}
\frac{1}{N}\sum_{i=1}^N R_i(T) \leq D\sqrt{T} + \frac{1}{2aN^2} \sum_{t=1}^T\bigg( \frac{1}{\sqrt{t}} \big\| \sum_{j=1}^N \mathbf{g}_j^t \big\|_*^2\bigg).
\end{equation}
\end{theorem}
\begin{proof}
Since $G$ is complete, at a fixed $t$, $\mathbf{w}_i^t$ is the same for any $i$. Hence $\mathbf{w}_i^{t+1} = \arg\min_{\mathbf{z}\in \mathcal{W}} \sum_{j=1}^{N_i} \left[ \eta_t \left\langle \mathbf{g}_j^t,\ \mathbf{z}-\mathbf{w}_j^t \right\rangle + \psi(\mathbf{z},\mathbf{w}_j^t) \right] = \arg\min_{\mathbf{z}\in\mathcal{W}} \big\langle \frac{\eta_t}{N}\sum_{j=1}^N \mathbf{g}_j^t,\ \mathbf{z}-\mathbf{w}_i^t \big\rangle + \psi(\mathbf{z}, \mathbf{w}_i^t) = P_{\mathbf{w}_i^t}(\frac{\eta_t}{N}\sum_{j=1}^N \mathbf{g}_j^t)$.
Let $\mathbf{u} = \frac{\eta_t}{N}\sum_{j=1}^N \mathbf{g}_j^t,\ \mathbf{v}=\mathbf{w}^*,\ \mathbf{w} = \mathbf{w}_i^t$ in (\ref{eq:lem_dol}), we have
\begin{equation*}
\begin{split}
&\big\langle \frac{1}{N}\sum_{j=1}^N \mathbf{g}_j^t,\ \mathbf{w}_i^t -\mathbf{w}^* \big\rangle\\
& \leq  \frac{1}{\eta_t} \left[ \psi(\mathbf{w}_i^t,\mathbf{w}^*) - \psi(\mathbf{w}_i^{t+1},\mathbf{w}^*)\right] + \frac{\eta_t}{2a}\big\|\frac{1}{N}\sum_{j=1}^N \mathbf{g}_j^t\big\|_*^2.
\end{split}
\end{equation*}
Using the convexity of $f_t$ and summing the above inequality over $t$ we have
\begin{equation*}
\begin{split}
&\frac{1}{N}\sum_{i=1}^N R_i(T) = \sum_{t=1}^T  \frac{1}{N}\sum_{j=1}^N \left[ f_j^t(\mathbf{w}_j^t) - f_j^t(\mathbf{w}^*) \right]  \\
& \leq \sum_{t=1}^T \big\langle \frac{1}{N}\sum_{j=1}^N \mathbf{g}_j^t,\ \mathbf{w}_i^t -\mathbf{w}^* \big\rangle \leq \frac{1}{\eta_1}\psi(\mathbf{w}_i^1, \mathbf{w}^*) - \\
& \frac{1}{\eta_T}\psi(\mathbf{w}_i^{T+1}, \mathbf{w}^*) + \sum_{2\leq t\leq T}\left(\frac{1}{\eta_t} - \frac{1}{\eta_{t-1}} \right) \psi(\mathbf{w}_i^t, \mathbf{w}^*) + \\
& \sum_{t=1}^T\bigg( \frac{\eta_t}{2aN^2} \big\| \sum_{j=1}^N \mathbf{g}_j^t \big\|_*^2\bigg).
\end{split}
\end{equation*}
Setting $\eta_t = 1/\sqrt{t}$ and using the assumption on the upper bound of $\psi(\cdot,\cdot)$ we reach the result.
\end{proof}

To appreciate the above theorem, we further assume that the subgradient is upper bounded: $\sup_{\mathbf{w}\in \mathcal{W},t=1,2,\ldots} \|g^t(\mathbf{w})\|_* = G$. In the most optimistic case, at a given round $t$, if the subgradients $\mathbf{g}_j^t,\ j=1,\ldots,N$ are mutually orthogonal, then the second term of the upper bound (\ref{eq:domd_bound}) can be bounded by $\frac{1}{2aN} G^2 \sqrt{T}$, which is $1/N$ times smaller than using a single agent. In the most pessimistic case, if all the subgradients $\mathbf{g}_j^t,\ j=1,\ldots,N$ are exactly the same, then the second term is bounded by $\frac{1}{2a} G^2 \sqrt{T}$, which is the same as in a single agent scheme.

According to the regret bound (\ref{eq:domd_bound}), the social regret $\sum_{i=1}^N R_i(T) \leq N D^2\sqrt{T} + \frac{1}{2aN} \sum_{t=1}^T\big( \eta_t \big\| \sum_{j=1}^N \mathbf{g}_j^t \big\|_*^2\big)$. In the most optimistic case, the bound is $N D^2\sqrt{T} + \frac{G^2}{2a}\sum_{t=1}^T\eta_t \leq (ND^2+\frac{G^2}{a})\sqrt{T}$. In the most pessimistic case, the bound becomes $(ND^2+\frac{NG^2}{a})\sqrt{T}$.

Imagine that $NT$ samples need to be processed. In the single agent scheme, they are accessed by only $1$ agent, while in the $N$ agents scheme, these $NT$ samples are evenly distributed with each $\mathcal{A}_i$ processing $T$ samples. In the most optimistic case, the bound for the $N$ agent scheme is $(ND^2+\frac{G^2}{a})\sqrt{T}$, while in the most pessimistic case, it is $(ND^2+N\frac{G^2}{a})\sqrt{T}$. In comparison, the bound for the single agent scheme is $(D^2\sqrt{N}+\frac{G^2\sqrt{N}}{a})\sqrt{T}$. We cannot draw an immediate conclusion of which one is better, since it depends on the correlations of examples, as well as $D$ and $G$. But it is clear that the $N$ agent scheme is at most $\sqrt{N}$ times larger in its social regret bound, while being $N$ times faster.


%

\section{Experimental Study}\label{sec:exp}
In this section, several sets of online classification experiments will be used to evaluate the theories and the proposed distributed online learning algorithms. Three real-world binary-class datasets \footnote{\url{www.csie.ntu.edu.tw/~cjlin/libsvmtools/datasets/}} from various application domains are adopted. Table \ref{tab:datasets} summarizes these datasets and the parameters used in section \ref{sec:exp_domd}.

\begin{table}[h]
\caption{Dataset facts and parameters.\label{tab:datasets}}
\setlength{\tabcolsep}{4pt}
{\small
\begin{tabular}{c|rrrr|rr}
    \hline
         Name        &  $\#$   & D   &Non-0 & Balance &C &S \\
    \hline
    \hline
        {svmguide1}   & 3,089    & 4   &100\%   & 1: 64.7$\%$& -- & --\\
        {cod-rna}     & 59,535   & 8   &100\% & -1: 66.7$\%$ & 1e-2 & 1e4\\
        {covtype}   & 522,910    & 54  &22\% & -1: 51.2$\%$  & 1e4  & 1e4\\
    \hline
\end{tabular}
}
\vskip -0.1in
\end{table}

To simulate the behavior of multi-agents, we use Pthreads (POSIX Threads) for multi-threaded programming, where each thread is an agent, and they communicate with each other via the shared memory. Barriers are used for synchronizations. All experiments are carried out on a workstation with a 4-core 2.7GHz Intel Core 2 Quad CPU.

\subsection{Distributed Weighted Majority}
To evaluate the proposed DWM algorithms, the simplest decision stumps are chosen as experts, and all the experts are trained off-line. We randomly choose $P\leq D$ dimensions. Within each dimension, $200$ probes are evenly placed between the min and max values of this dimension. The probe with the minimum training error over the whole dataset is selected as the decision threshold. In all the following weighted majority experiments, we choose the penalty factor $\alpha = 0.9$.

The first set of experiments report the behaviors of \textsf{DWM-I} and \textsf{DWM-A} from the individual agent point of view. Each agent share the same $P=4$ experts, and communicates with all the others. Fig. \ref{fig:DWM_I_svmguide1_1234_threads} and \ref{fig:DWM_A_svmguide1_1234_threads} depict the cumulative number of mispredictions made by each thread as a function of the number of samples accessed by a single agent, where $1$, $2$, $3$ and $4$ agents are compared. Each plot in a subfigure represents an agent. It is clear that an agent $\mathcal{A}_i$ makes fewer mistakes $M_i$ as it receives more information from its neighbors. With $4$ agents, $M_i$ is reduced by half comparing with the single agent case. This provides some evidence for the $1/N$ error reduction as stated in Theorem \ref{thm:dist_wm_imit}.
\begin{figure}[h!]
\begin{center}
\includegraphics[width=1.0\columnwidth]{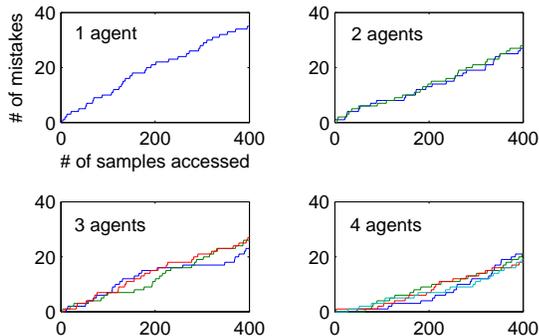}
\caption{\textsf{DWM-I}: cumulative mistakes on svmguide1.}
\label{fig:DWM_I_svmguide1_1234_threads}
\end{center}
\vskip -0.2in
\end{figure}
\begin{figure}[h!]
\begin{center}
\includegraphics[width=1.0\columnwidth]{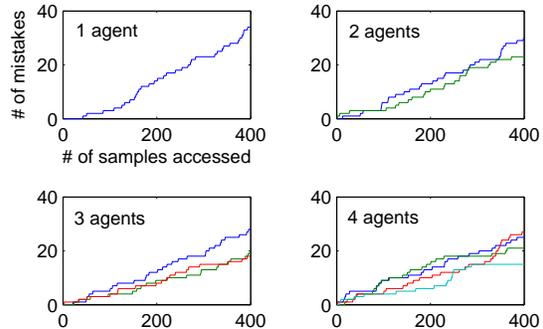}
\caption{\textsf{DWM-A}: cumulative mistakes on svmguide1.}
\label{fig:DWM_A_svmguide1_1234_threads}
\end{center}
\vskip -0.2in
\end{figure}

As discussed in Section \ref{sec:dwm}, from the social point of view, with the same number of samples accessed, the bound (\ref{eq:dwm_total_bound}) of the total number of mistakes made by all agents ($\sum_i M_i$) is almost as small as that in a single agent case. The comparisons for both \textsf{DWM-I} and \textsf{DWM-A} are illustrated in Fig.\ref{fig:DWM_IA_svmguide1_total_misp}. This result is not surprising, since no more information is provided for multiple agents, and one should not hope that $\sum_i M_i$ is much lower than $M$. But on the other hand, the DWM algorithms achieve the same level of mistakes, while they are $N$ times faster. It can also be observed from Fig. \ref{fig:DWM_I_svmguide1_1234_threads}, \ref{fig:DWM_A_svmguide1_1234_threads} and  \ref{fig:DWM_IA_svmguide1_total_misp} that \textsf{DWM-A} makes slightly fewer mistakes than \textsf{DWM-I}.
\begin{figure}[h!]
\begin{center}
\includegraphics[width=1.0\columnwidth]
{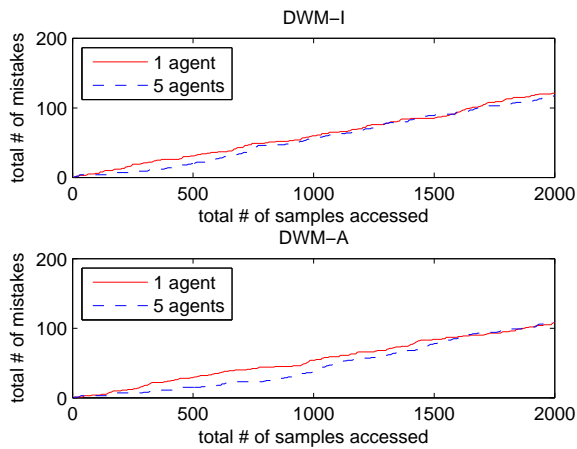}
\caption{svmguide1: total $\#$ of mistakes over all agents.}
\label{fig:DWM_IA_svmguide1_total_misp}
\end{center}
\vskip -0.2in
\end{figure}

To verify the tightness of bound (\ref{eq:dist_wm_imit_1}) and the refined (\ref{eq:dist_wm_avg_3}), we compare the number of mistakes $m_*$ make by the best expert $E_*$ over all agents with that of a single agent $M_i$. Fig. \ref{fig:DWM_IA_svmguide1_best_expert} shows that with $N=2$ or $5$ agents, $m_*$ is around $2$ or $5$ times larger than $M_i$, which means $M_i \approx m_*/N$. However, choosing $\alpha=0.9$ in bound (\ref{eq:dist_wm_imit_1}) leads to $M_i \leq 2.05 m_*/N$. This shows that the bound in Theorem \ref{thm:dist_wm_avg} is indeed tighter than Theorem \ref{thm:dist_wm_imit}.
\begin{figure}[h!]
\begin{center}
\includegraphics[width=1.0\columnwidth]{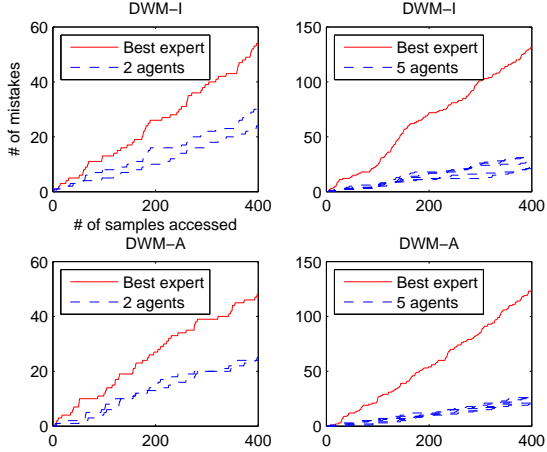}
\caption{svmguide1; $\#$ of mistakes: best expert v.s. multi-agents.}
\label{fig:DWM_IA_svmguide1_best_expert}
\end{center}
\vskip -0.2in
\end{figure}


%
%

\subsection{Distributed Online Mirror Descent}\label{sec:exp_domd}
In this section, several online classification experiments will be carried out using the proposed \textsf{DOGD} and \textsf{DOEG} algorithms. For \textsf{DOGD}, we choose the L2-regularized instance hinge loss function as our convex objective function:
\begin{equation*}
f_t(\mathbf{w}) = C\max\left\{0, 1- l_t \mathbf{w}^T \mathbf{x}_t \right\} + \|\mathbf{w}\|_2^2/2.
\end{equation*}
For \textsf{DOEG}, we take
$f_t(\mathbf{w}) = \max\left\{0, 1- l_t \mathbf{w}^T \mathbf{x}_t  \right\}$
and $\mathcal{W}= \{\mathbf{w}:\|\mathbf{w}\|_1 \leq S\}$. Since the update rule (\ref{eq:DOEG}) cannot change the signs of $\mathbf{w}_t$, we use a similar trick like $EG^{\pm}$ proposed in \cite{kivinen97eg}, i.e. letting $\mathbf{w}=\mathbf{w}^+ - \mathbf{w}^-$, where $\mathbf{w}^+,\mathbf{w}^- > \mathbf{0}$. Since we will not compare the generalization capacities between these two algorithms, in all the following experiments, the parameters of $C$ and $S$ are chosen according to Table \ref{tab:datasets} without any further tuning. The subgradient of the non-smooth hinge loss is take as $\mathbf{g}_t =-l_t \mathbf{x}_t$ if $1- l_i \mathbf{w}^T \mathbf{x}_t > 0$ and $0$ otherwise.

We firstly illustrate the generalization capacities of \textsf{DOGD} and \textsf{DOEG}. Since we do not know $\inf_{\mathbf{w}\in\mathcal{W}}\sum_{t=1}^T f_t(\mathbf{w})$, it is not easy to calculate the individual regret or social regret. Hence we only compare the number of mispredictions and the average accumulated objective function values as functions of the number of samples accessed by a single agent. The results are shown in Fig. \ref{fig:DOGD_misp_cod_14816_threads} $\sim$ \ref{fig:DOEG_obj_cod_12481632_threads}.
\begin{figure}[h!]
\begin{center}
\includegraphics[width=0.9\columnwidth]
{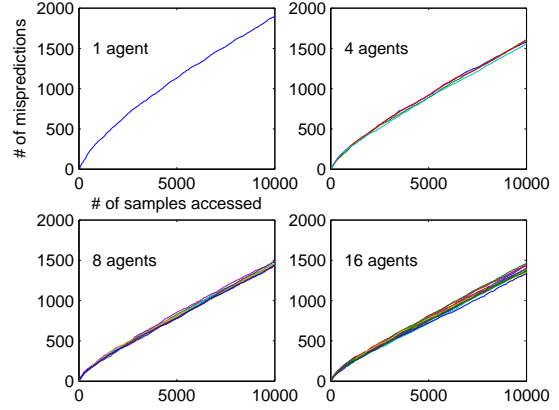}
\caption{DOGD; cod-rna; mispredictions}
\label{fig:DOGD_misp_cod_14816_threads}
\end{center}
\vskip -0.2in
\end{figure}
\begin{figure}[h!]
\begin{center}
\includegraphics[width=1.0\columnwidth]
{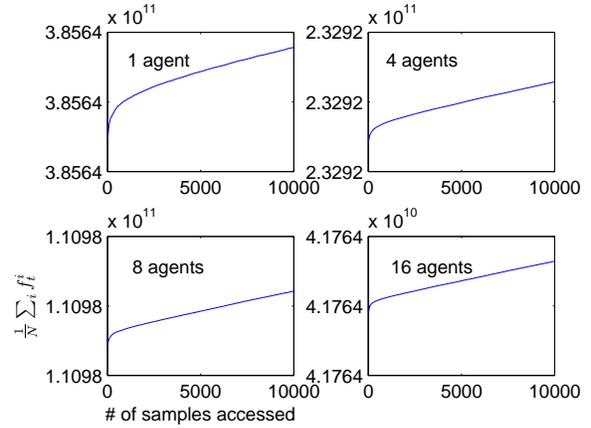}
\caption{DOGD; cod-rna; average objective values}
\label{fig:DOGD_obj_cod_14816_threads}
\end{center}
\vskip -0.2in
\end{figure}

It is clear that for both \textsf{DOGD} and \textsf{DOEG}, the number of mispredictions decreases when more agents communicate with each other. The average objective values $\frac{1}{N}f_t^i(\mathbf{w}_t^i)$ also decrease with the increasing number of agents $N$. However, as shown in Fig. \ref{fig:DOEG_obj_cod_12481632_threads}, when $N=32$, the averaged $\frac{1}{N}f_t^i(\mathbf{w}_t^i)$ is larger than $N=16$. This might be due to the insufficient number of samples of the dataset cod-rna. This conjecture is experimentally verified in Fig. \ref{fig:DOEG_obj_covtype_12481632_threads}, where the size of covtype is $522910$.

\begin{figure}[h!]
\begin{center}
\includegraphics[width=0.9\columnwidth]
{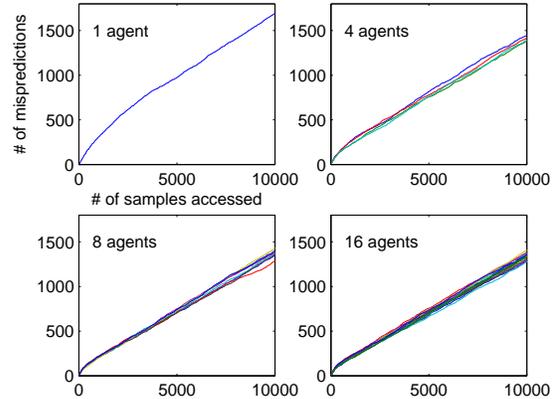}
\caption{DOEG; cod-rna; mispredictions}
\label{fig:DOEG_misp_cod_14816_threads}
\end{center}
\vskip -0.2in
\end{figure}
\begin{figure}[h!]
\begin{center}
\includegraphics[width=0.8\columnwidth]
{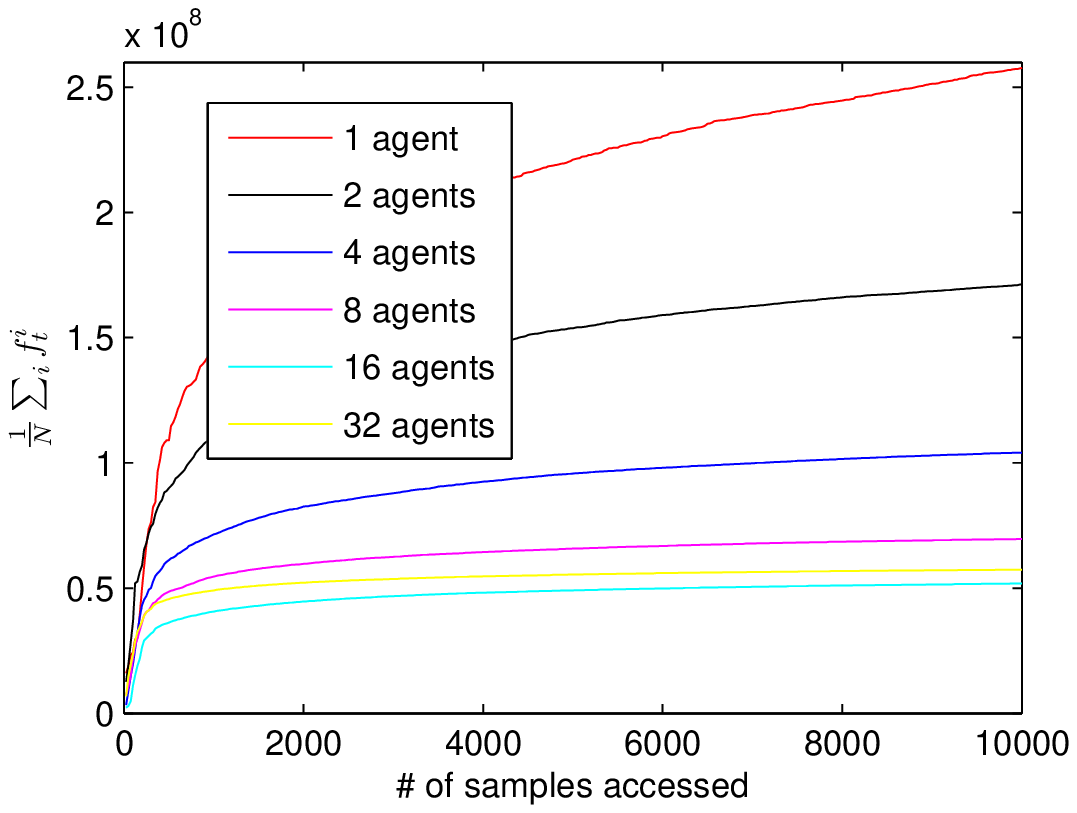}
\caption{DOEG; cod-rna; average objective values}
\label{fig:DOEG_obj_cod_12481632_threads}
\end{center}
\vskip -0.2in
\end{figure}
\begin{figure}[h!]
\begin{center}
\includegraphics[width=0.8\columnwidth]
{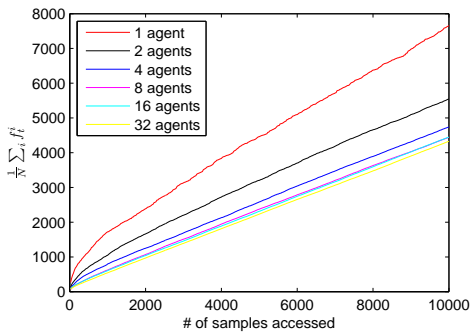}
\caption{DOEG; covtype; average objective values}
\label{fig:DOEG_obj_covtype_12481632_threads}
\end{center}
\vskip -0.2in
\end{figure}

As discussed at the end of Section \ref{sec:DOMD}, the social regret bound of $N$ agents is at most $\sqrt{N}$ times larger than that of a single agent scheme. The next set of experiments will be used to verify this claim. Fig. \ref{fig:DOEG_obj_covtype_total} depicts the result.
\begin{figure}[h!]
\begin{center}
\includegraphics[width=0.9\columnwidth]
{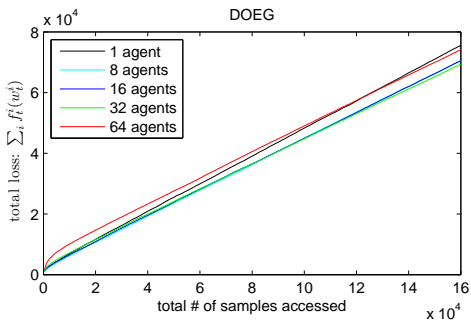}
\caption{DOEG; covtype; total objective values}
\label{fig:DOEG_obj_covtype_total}
\end{center}
\vskip -0.2in
\end{figure}
We can see that the total loss $\sum_{i=1}^N f_t^i(\mathbf{w}_t^i)$ for $N=8,16,32$ is even lower than using a single agent. $N=64$ is slightly higher, but the difference is still much lower than the theoretical $\sqrt{64}$. This suggests that there might exist a bound tighter than (\ref{eq:domd_bound}).




%

\section{Conclusions and Future Work}
We proposed a generic data-distributed online learning meta-algorithm. As concrete examples, two sets of distributed algorithms were derived. One is for distributed weighted majority, and the other is for distributed online convex optimization. Their effectiveness is supported by both analysis and experiments.

The analysis shows that with $N$ agents, DWM can have an upper error bound that is $1/N$ lower than using a single agent. From the social point of view, the bound of total number of errors made by all $N$ agents is the same as using $1$ agent, while processing the same amount of examples. This indicates that DWM attains the same level of generalization error as WM, but is $N$ times faster.

The average individual regret for DOMD algorithms is also much lower than OMD, although it is not $1/N$ lower as in DWM. In the worst case, the bound of social regret is at most $\sqrt{N}$ higher than using a single agent.

In follow-on work, two assumptions made in this paper will be removed to make the proposed algorithms more robust in practical applications. Firstly, as discussed in \cite{duchi10}, the connected graph $G$ does not need to be complete. We are working on distributed active learning and active teaching, which might lead to a data-dependent communication topology. Secondly, the learning process should be fully asynchronous. This brings up the problem of `delays' in label feedbacks \cite{mesterharm05oldlf,langford09slaf}. Moreover, for OCO, with more structural information on $f_t$ rather than the black-box model, we might be able to find better distributed algorithms and achieve tighter bounds.

\bibliographystyle{icml2011}
{\small
\bibliography{Hua_ICML11}
}

\end{document}